\documentclass[11pt]{article}
\usepackage{fullpage}
\usepackage{paralist}
\usepackage{xcolor}         

\usepackage{natbib}

\usepackage[utf8]{inputenc} 
\usepackage[T1]{fontenc}    
\usepackage{url}            
\usepackage{booktabs}       
\usepackage{amsfonts}       
\usepackage{nicefrac}       
\usepackage{microtype}      

\usepackage{ulem}
\usepackage{graphicx,wrapfig,lipsum}
\usepackage{float}    
\usepackage{verbatim} 
\usepackage{amsmath, amssymb, amsthm, mathtools}  
\usepackage{subfig}   
\usepackage{bookmark}
\usepackage{enumerate}
\usepackage{paralist}
\usepackage{xspace}
\usepackage{caption}
\usepackage{bbm}
\usepackage{bm}
\usepackage{lipsum}
\usepackage{algorithm}
\usepackage{algorithmic}

\newtheorem{theorem}{Theorem}
\newtheorem{definition}{Definition}
\newtheorem{lemma}{Lemma}

\newtheorem{assumption}{Assumption}

\usepackage{amsmath,amsfonts,bm}

\def\1{\bm{1}}

\def\rvv{{\mathbf{v}}}

\def\rmD{{\mathbf{D}}}

\def\rmX{{\mathbf{X}}}
\def\rmY{{\mathbf{Y}}}

\DeclareMathAlphabet{\mathsfit}{\encodingdefault}{\sfdefault}{m}{sl}
\SetMathAlphabet{\mathsfit}{bold}{\encodingdefault}{\sfdefault}{bx}{n}

\newcommand{\R}{\mathbb{R}}

\newcommand{\bD}{\mathbf{D}}
\newcommand{\D}{\mathbf{D}}
\newcommand{\bA}{\mathbf{A}}

\newcommand{\bU}{\mathbf{U}}
\newcommand{\bV}{\mathbf{V}}
\newcommand{\bSigma}{\mathbf{\Sigma}}
\newcommand{\bX}{\mathbf{X}}
\newcommand{\bY}{\mathbf{Y}}
\newcommand{\bI}{\mathbf{I}}
\newcommand{\bP}{\mathbf{P}}
\newcommand{\bPi}{\mathbf{\Pi}}

\newcommand{\bx}{\mathbf{x}}
\newcommand{\by}{\mathbf{y}}

\newcommand{\bbracket}[1]{\left(#1\right)}
\newcommand{\bip}[1]{\left\langle#1\right\rangle}
\newcommand{\bnorm}[1]{\left\|#1\right\|}

\title{Personalized Dictionary Learning for Heterogeneous Datasets
}
\author{
	Geyu Liang\\
	Industrial and Operations Engineering\\
	University of Michigan\\
	\texttt{lianggy@umich.edu}
	\and
	Naichen Shi\\
	Industrial and Operations Engineering\\
	University of Michigan\\
	\texttt{naichens@umich.edu}
	\and
	Raed Al Kontar\\
	Industrial and Operations Engineering\\
	University of Michigan\\
	\texttt{alkontar@umich.edu}
	\and 
	Salar Fattahi\\
	Industrial and Operations Engineering\\
	University of Michigan\\
	\texttt{fattahi@umich.edu}
}

\begin{document}
\maketitle

\begin{abstract}
   We introduce a relevant yet challenging problem named \textit{Personalized Dictionary Learning} ($\mathsf{PerDL}$), where the goal is to learn sparse linear representations from heterogeneous datasets that share some commonality. In $\mathsf{PerDL}$, we model each dataset's shared and unique features as \textit{global} and \textit{local} dictionaries. Challenges for $\mathsf{PerDL}$ not only are inherited from classical \textit{dictionary learning} (DL), but also arise due to the unknown nature of the shared and unique features. In this paper, we rigorously formulate this problem and provide conditions under which the global and local dictionaries can be provably disentangled. Under these conditions, we provide a meta-algorithm called \textit{Personalized Matching and Averaging} ($\mathsf{PerMA}$) that can recover both global and local dictionaries from heterogeneous datasets. $\mathsf{PerMA}$ is highly efficient; it converges to the ground truth at a linear rate under suitable conditions. Moreover, it automatically borrows strength from strong learners to improve the prediction of weak learners. As a general framework for extracting global and local dictionaries, we show the application of $\mathsf{PerDL}$ in different learning tasks, such as training with imbalanced datasets and video surveillance. 
\end{abstract}
\section{Introduction}
Given a set of $n$ signals $\bY = [\by_1,\dots,\by_n]\in \R^{d\times n}$, \textit{dictionary learning} (DL) aims to find a \textit{dictionary} $\bD\in\R^{d\times r}$ and a corresponding \textit{code} $\bX = [\bx_1,\dots,\bx_n]\in \R^{r\times n}$ such that:  (1) each data sample $\by_i$ can be written as $\by_i=\bD\bx_i$ for $1\le i\le n$, and (2) the code $\bX$ has as few nonzero elements as possible. The columns of the dictionary $\bD$, also known as \textit{atoms}, encode the ``common features'' whose linear combinations form the data samples.
A typical approach to solve DL is via the following optimization problem:
\begin{equation}
    \min_{\bX,\bD} \|\bY-\bD\bX\|_F^2+\lambda\|\bX\|_{\ell_q},
\tag{DL}\label{objective DL}
\end{equation}
Here $\|\cdot\|_{\ell_q}$ is often modeled as a $\ell_1$-norm \citep{arora2015simple,agarwal2016learning} or $\ell_0$-(pseudo-)norm \citep{spielman2012exact,liang2022simple} and has the role of promoting sparsity in the estimated sparse code. Due to its effective feature extraction and representation, DL has found immense applications in data analytics, with applications ranging from clustering and classification (\cite{ramirez2010classification,tovsic2011dictionary}), to image denoising (\cite{li2011efficient}), to document detection(\cite{kasiviswanathan2012online}), to medical imaging (\cite{zhao2021survey}), and to many others. 

However, the existing formulations of DL hinge on a critical assumption: the homogeneity of the data. It is assumed that the samples share the {same} set of features (atoms) collected in a \textit{single} dictionary $D$. This assumption, however, is challenged in practice as the data is typically collected and processed in heterogeneous edge devices (clients). These clients (for instance, smartphones and wearable devices) operate in different conditions~\citep{kontar2017nonparametric} while sharing some congruity. Accordingly, the collected datasets are naturally endowed with heterogeneous features while potentially sharing common ones. 
In such a setting, the classical formulation of DL faces a major dilemma: on the one hand, a reasonably-sized dictionary (with a moderate number of atoms $r$) may overlook the unique features specific to different clients. On the other hand, collecting both shared and unique features in a single enlarged dictionary (with a large number of atoms $r$) may lead to computational, privacy, and identifiability issues. In addition, both approaches fail to provide information about ``what is shared and unique'' which may offer standalone intrinsic value and can potentially be exploited for improved clustering, classification and anomaly detection, amongst others. 

With the goal of addressing data heterogeneity in dictionary learning, in this paper, we propose \textit{personalized dictionary learning} ($\mathsf{PerDL}$); a framework that can untangle and recover \textit{global} and \textit{local (unique)} dictionaries from heterogeneous datasets. The global dictionary, which collects atoms that are shared among all clients, represents the common patterns among datasets and serves as a conduit of collaboration in our framework. The local dictionaries, on the other hand, provide the necessary flexibility for our model to accommodate data heterogeneity. 

We summarize our contributions below:

\begin{itemize}
    \item [-] {\it Identifiability of local and global atoms:} We provide conditions under which the local and global dictionaries can be provably identified and separated by solving a nonconvex optimization problem. At a high level, our identifiability conditions entail that the true dictionaries are column-wise incoherent, and the local atoms do not have a significant alignment along any nonzero vector. 
    \item [-] {\it Federated meta-algorithm:} We present a fully federated meta-algorithm, called $\mathsf{PerMA}$ (Algorithm~\ref{alg: general alg}), for solving $\mathsf{PerDL}$. $\mathsf{PerMA}$ only requires communicating the estimated dictionaries among the clients, thereby circumventing the need for sharing any raw data. A key property of $\mathsf{PerMA}$ is its ability to untangle global and local dictionaries by casting it as a series of shortest path problems over a \textit{directed acyclic graph} (DAG). 
    \item [-] {\it Theoretical guarantees:}  We prove that, under moderate conditions on the generative model and the clients, $\mathsf{PerMA}$ enjoys a linear convergence to the ground truth up to a statistical error. Additionally, $\mathsf{PerMA}$ borrows strength from strong clients to improve the performance of the weak ones. More concretely, through collaboration, our framework provides weak clients with the extra benefits of \textit{averaged} initial condition, convergence rate, and final statistical error.
    \item [-] {\it Practical performance:} We showcase the performance of $\mathsf{PerMA}$ on a synthetic dataset, as well as different realistic learning tasks, such as training with imbalanced datasets and video surveillance. 
    These experiments highlight that our method can effectively extract shared global features while preserving unique local ones, ultimately improving performance through collaboration. 
\end{itemize}

\subsection{Related Works}
\paragraph{Dictionary Learning} 
\cite{spielman2012exact, liang2022simple} provide conditions under which \ref{objective DL} can be provably solved, provided that the dictionary is a square matrix (also known as \textit{complete} DL). For the more complex case of \textit{overcomplete} DL with $r>d$, \cite{arora2014new,arora2015simple,agarwal2016learning} show that alternating minimization achieves desirable statistical and convergence guarantees.
Inspired by recent results on the benign landscape of matrix factorization~\citep{ge2017no, fattahi2020exact},~\cite{sun2016complete} show that a smoothed variant of \ref{objective DL} is devoid of spurious local solutions. In contrast, distributed or federated variants of \ref{objective DL} are far less explored. \cite{huang2022federated,gkillas2022federated} study \ref{objective DL} in the federated setting. However, they do not provide any provable guarantees on their proposed method.

\paragraph{Federated Learning \& Personalization} Recent years have seen explosive interest in federated learning (FL) following the seminal paper on federated averaging \citep{mcmahan2017communication}. Literature along this line has primarily focused on predictive modeling using deep neural networks (DNN), be it through enabling faster convergence \citep{karimireddy2020scaffold}, improving aggregation schemes at a central server \citep{wang2020federated}, promoting fairness across all clients \citep{yue2022gifair} or protecting against potential adversaries \citep{bhagoji2019analyzing}. More recently, the focus has been placed on tackling heterogeneity across client datasets through personalization. The key idea is to allow each client to retain their own tailored model instead of learning one model that fits everyone. Approaches along this line either split weights of a DNN into shared and unique ones and collaborate to learn the shared weights \citep{liang2020think, t2020personalized}, or follow a train-then-personalize approach where a global model is learned and fine-tuned locally, often iteratively \citep{li2021ditto}. Again such models have mainly focused on predictive models. Whilst this literature abounds, personalization that aims to identify what is shared and unique across datasets is very limited. Very recently, personalized PCA \citep{personalizedpca} was proposed to address this challenge through identifiably extracting shared and unique principal components using distributed Riemannian gradient descent.  However, PCA cannot accommodate sparsity in representation and requires orthogonality constraints that may limit its application. In contrast, our work considers a broader setting via sparse dictionary learning.

\paragraph{Notation.}
For a matrix $\bA$, we use $\|\bA\|_2$, $\|\bA\|_F$, $\|\bA\|_{1,2}$, and $\|\bA\|_1$ to denote its spectral norm, Frobenius norm, the maximum of its column-wise 2-norm, and the element-wise 1-norm of $\bA$, respectively. We use $\bA_i$ to indicate that it belongs to client $i$. Moreover, we use $\bA_{(i)}$ to denote the $i$-th column of $\bA$. We use $\mathcal{P}(n)$ to denote the set of $n \times n$ signed permutation matrices. We define $[n] = \{1,2,\dots, n\}$.  

\section{$\mathsf{PerDL}$: Personalized Dictionary Learning}\label{sec: model}
In $\mathsf{PerDL}$, we are given $N$ clients, each with $n_i$ samples collected in $\bY_i\in \R^{d\times n_i}$ and generated as a sparse linear combination of $r^g$ global atoms and $r_i^l$ local atoms:
\begin{equation}\label{model: PerDL}
    \bY_i = \bD^*_i\bX^*_i,\quad \text{where}\quad \bD^*_i = \begin{bmatrix}
        \bD^{g*}&\bD^{l*}_i
    \end{bmatrix}, \quad \text{for}\quad  1=1,\dots,N.
\end{equation}
Here $\bD^{g*}\in \R^{d\times r^g}$ denotes a global dictionary that captures the common features shared among all clients, whereas $\bD^{l*}_i\in \R^{d\times r^l_i}$ denotes the local dictionary specific to each client. Let $r_i=r^g+r^l_i$ denote the total number of atoms in $\bD^*_i$. Without loss of generality, we assume the columns of $\bD^*_i$ have unit $\ell_2$-norm.\footnote{This assumption is without loss of generality since, for any dictionary-code pair $(\bD_i,\bX_i)$, the columns of $\bD_i$ can be normalized to have unit norm by re-weighting the corresponding rows of $\bX_i$.} The goal in $\mathsf{PerDL}$ is to recover $\bD^{g*}$ and $\{\bD^{l*}_i\}_{i=1}^N$, as well as the sparse codes $\{\bX^*_i\}_{i=1}^N$, given the datasets $\{\bY_i\}_{i=1}^N$. Before presenting our approach for solving $\mathsf{PerDL}$, we first consider the following fundamental question: under what conditions is the recovery of the dictionaries $\bD^{g*},\{\bD^{l*}_i\}_{i=1}^N$ and sparse codes $\{\bX^*_i\}_{i=1}^N$ well-posed? 

To answer this question, we first note that it is only possible to recover the dictionaries and sparse codes up to a signed permutation: given any signed permutation matrix $\bPi\in \mathcal{P}(r_i)$, the dictionary-code pairs $(\bD_i, \bX_i)$ and $(\bD_i\bPi_i, \bPi_i^\top\bX_i)$ are equivalent. This invariance with respect to signed permutation gives rise to an equivalent class of true solutions with a size that grows exponentially with the dimension. To guarantee the recovery of a solution from this equivalent class, we need the $\mu$-incoherency of the true dictionaries.
\begin{assumption}[$\mu$-incoherency]\label{def: incoherency} For each client $1\le i \le N$, the dictionary $\rmD^*_i$ is $\mu$-incoherent for some constant $\mu>0$, that is,
\begin{equation}
\max_{j,k}\left|\left\langle\left(\rmD^*_i\right)_{(j)},\left(\rmD^*_i\right)_{(k)}\right\rangle\right|\le \frac{\mu}{\sqrt{d}}.
\end{equation}
\end{assumption}
Assumption~\ref{def: incoherency} is standard in dictionary learning (\cite{agarwal2016learning,arora2015simple,chatterji2017alternating}) and was independently introduced by \cite{fuchs2005recovery,tropp2006just} in signal processing and \cite{zhao2006model, meinshausen2006high} in statistics. Intuitively, it requires the atoms in each dictionary to be approximately orthogonal. To see the necessity of this assumption, consider a scenario where two atoms of $\bD^*_i$ are perfectly aligned (i.e., $\mu=\sqrt{d}$). In this case, using either of these two atoms achieve a similar effect in reconstructing $\bY_i$, contributing to the ill-posedness of the problem.

Our next assumption guarantees the separation of local dictionaries from the global one in $\mathsf{PerDL}$. First, we introduce several signed permutation-invariant distance metrics for dictionaries, which will be useful for later analysis.
\begin{definition}
    For two dictionaries $\bD_1,\bD_2\in\R^{d\times r}$, we define their signed permutation-invariant $\ell_{1,2}$-distance and $\ell_2$-distance as follows:
    \begin{align}
        d_{1,2}(\bD_1,\bD_2) &:= \min_{\bPi\in\mathcal{P}(r)} \|\bD_1\bPi-\bD_2\|_{1,2},\\
        d_{2}(\bD_1,\bD_2) &:= \min_{\bPi\in\mathcal{P}(r)} \|\bD_1\bPi-\bD_2\|_{2}.
    \end{align}
    Furthermore, suppose $\bPi^* = \arg\min_{\bPi\in\mathcal{P}(r)} \|\bD_1\bPi-\bD_2\|_{1,2}$. For any $1\le j \le r$, we define
    \begin{equation}
        d_{2,(j)}(\bD_1,\bD_2) := \bnorm{\bbracket{\bD_1\bPi^*-\bD_2}_{(j)}}_2.
    \end{equation}
\end{definition}
\begin{assumption}[$\beta$-identifiablity]\label{def: identifiablity}
The local dictionaries $\left\{\rmD^{l*}_i\right\}_{i=1}^N$ are $\beta$-identifiable for some constant $0<\beta<1$, that is, there exists no vector $\rvv\in\R^d$ with $\|\rvv\|_2=1$ such that 
\begin{equation}\label{eq: condition no global candidate}
    \max_{1\le i\le N}\min_{1\le j \le r_l}d_2\bbracket{\left(\rmD^{l*}_i\right)_{(j)} , \rvv }\le \beta.
\end{equation}
\end{assumption}
Suppose there exists a unit-norm $\rvv$ satisfying~\eqref{eq: condition no global candidate} for some small $\beta>0$. This implies that $\rvv$ is sufficiently close to at least one atom from each local dictionary. Indeed, one may treat this atom as part of the global dictionary, thereby violating the identifiability of local and global dictionaries. On the other hand, the infeasibility of ~\eqref{eq: condition no global candidate} for large $\beta>0$ implies that the local dictionaries are sufficiently dispersed, which in turn facilitates their identification.   

With the above assumptions in place, we are ready to present our proposed optimization problem for solving $\mathsf{PerDL}$:
\begin{equation}\tag{$\texttt{PerDL-NCVX}$}\label{objective PerDL}
    \begin{aligned}
        \min_{\bD^g,\{\bD_i\},\{\bX_i\}}\sum_{i=1}^N\|\bY_i-\bD_i\bX_i\|_F^2+\lambda\sum_{i=1}^N\|\bX_i\|_{\ell_q},\quad 
    \mathrm{s.t.}\ \ \bbracket{\bD_i}_{(1:r^g)} = \bD^g\quad\mathrm{for}\quad 1\le i\le N.
    \end{aligned}
\end{equation}
For each client $i$, \ref{objective PerDL} aims to recover a dictionary-code pair $(\bD_i, \bX_i)$ that match $\bY_i$ under the constraint that dictionaries for individual clients share the same global components.

\section{Meta-algorithm of Solving PerDL}\label{sec: framework}
In this section, we introduce our meta-algorithm (Algorithm~\ref{alg: general alg}) for solving \ref{objective PerDL}, which we call \textit{Personalized Matching and Averaging} ($\mathsf{PerMA}$).
\begin{algorithm}
	\caption{$\mathsf{PerMA}$: Federated Matching and Averaging}
	\label{alg: general alg}
	\begin{algorithmic}[1]
		\STATE{{\bf Input:} $\{\bY_i\}_{i=1}^N$.}
            \FOR{client $i=1,...,N$}
            \STATE {\it Client:} Obtain $\bD^{(0)}_i$ based on $\bY_i$. $\hfill\color{cyan} \texttt{// Initialization step}$\label{step: local initialization}
            \STATE {\it Client:} Send $\bD^{(0)}_i$ to the server.
            \ENDFOR
            \STATE{{\it Server:} $\left(\bD^{g,(0)}, \{\bD^{l,(0)}_i\}_{i=1}^N\right) = \texttt{global\_matching}\left(\{\bD^{(0)}_i\}_{i=1}^N\right)$}\\$\hfill\color{cyan} \texttt{// Separating local from global dictionaries}$\label{step: global matching}
            \STATE {\it Server:} Broadcast $\left(\bD^{g,(0)}, \{\bD^{l,(0)}_i\}_{i=1}^N\right)$
		\FOR{$t = 0,1,\dots, T$}
                \FOR{client $i=1,...,N$}
                \STATE{ {\it Client:} $\left(\bD^{g,(t+1)}_i, \bD^{l,(t+1)}_i\right) = \texttt{local\_update}\left(\bY_i, \bD^{g,(t)}, \bD^{l,(t)}_i\right)$}\label{step: local dictionary update}\\$\hfill\color{cyan} \texttt{// Updating the local and global dictionaries for each client}$
                \STATE{{\it Client:} Send $\bD^{g,(t+1)}_i$ to the server.}
                \ENDFOR
		      \STATE{{\it Server:} Calculate $\bD^{g,(t+1)}=\frac{1}{N}\sum_{i=1}^N\bD^{g,(t+1)}_i$.}\label{step: global aggregation}$\hfill\color{cyan} \texttt{// Averaging global dictionaries}$
                \STATE{{\it Server:} Broadcast $\bD^{g,(t+1)}$.}
		\ENDFOR
		\RETURN{$\left(\bD^{g,(T)},\{\bD^{l,(T)}_i\}_{i=1}^N\right)$}. 
	\end{algorithmic}
\end{algorithm}
In what follows, we explain the steps of $\mathsf{PerMA}$:
\paragraph{Local initialization (Step~\ref{step: local initialization}):} $\mathsf{PerMA}$ starts with a warm-start step where each client runs their own initialization scheme to obtain $\bD^{(0)}_i$. This step is necessary even for the classical DL to put the initial point inside a basin of attraction of the ground truth. Several spectral methods were proposed to provide a theoretically good initialization\citep{arora2015simple,agarwal2016learning}, while in practice,  it is reported that random initialization followed by a few iterations of alternating minimization approach will suffice \citep{ravishankar2020analysis,liang2022simple}.
\paragraph{Global matching scheme (Step~\ref{step: global matching})} Given the clients' initial dictionaries, our global matching scheme separates the global and local parts of each dictionary by solving a series of shortest path problems on an auxiliary graph. Then, it obtains a refined estimate of the global dictionary via simple averaging. A detailed explanation of this step is provided in the next section. 
\paragraph{Dictionary update at each client (Step~\ref{step: local dictionary update})} During each communication round, the clients refine their own dictionary based on the available data, the aggregated global dictionary, and the previous estimate of their local dictionary. A detailed explanation of this step is provided in the next section.
\paragraph{Global aggregation (Step~\ref{step: global aggregation})} At the end of each round, the server updates the clients' estimate of the global dictionary by computing their average. 

A distinguishing property of $\mathsf{PerMA}$ is that it only requires the clients to communicate their dictionaries and not their sparse codes. In fact, after the global matching step on the initial dictionaries, the clients only need to communicate their global dictionaries, keeping their local dictionaries private. 

\subsection{Global Matching and Local Updates}\label{subsec: selection}
In this section, we provide detailed implementations of \texttt{global\_matching} and \texttt{local\_update} subroutines in $\mathsf{PerMA}$ (Algorithm~\ref{alg: general alg}).

Given the initial approximations of the clients' dictionaries $\{\bD_i^{(0)}\}_{i=1}^N$, \texttt{global\_matching} seeks to identify and aggregate the global dictionary by extracting the similarities among the atoms of $\{\bD_i^{(0)}\}$. To identify the global components, one approach is to solve the following optimization problem
\begin{equation}\label{opt_matching}
    \begin{aligned}
        \min_{\bPi_{i}} \sum_{i=1}^{N-1}\left\|\left(\bD_i^{(0)}\bPi_i\right)_{(1:r^g)}-\left(\bD_{i+1}^{(0)}\bPi_{i+1}\right)_{(1:r^g)}\right\|_2\quad 
        \mathrm{s.t.}\quad & \bPi_i\in\mathcal{P}(r_i)\quad\mathrm{for}\quad 1\le i\le N.
    \end{aligned}
\end{equation}

\begin{figure}
    \centering
        \includegraphics[width=\textwidth]{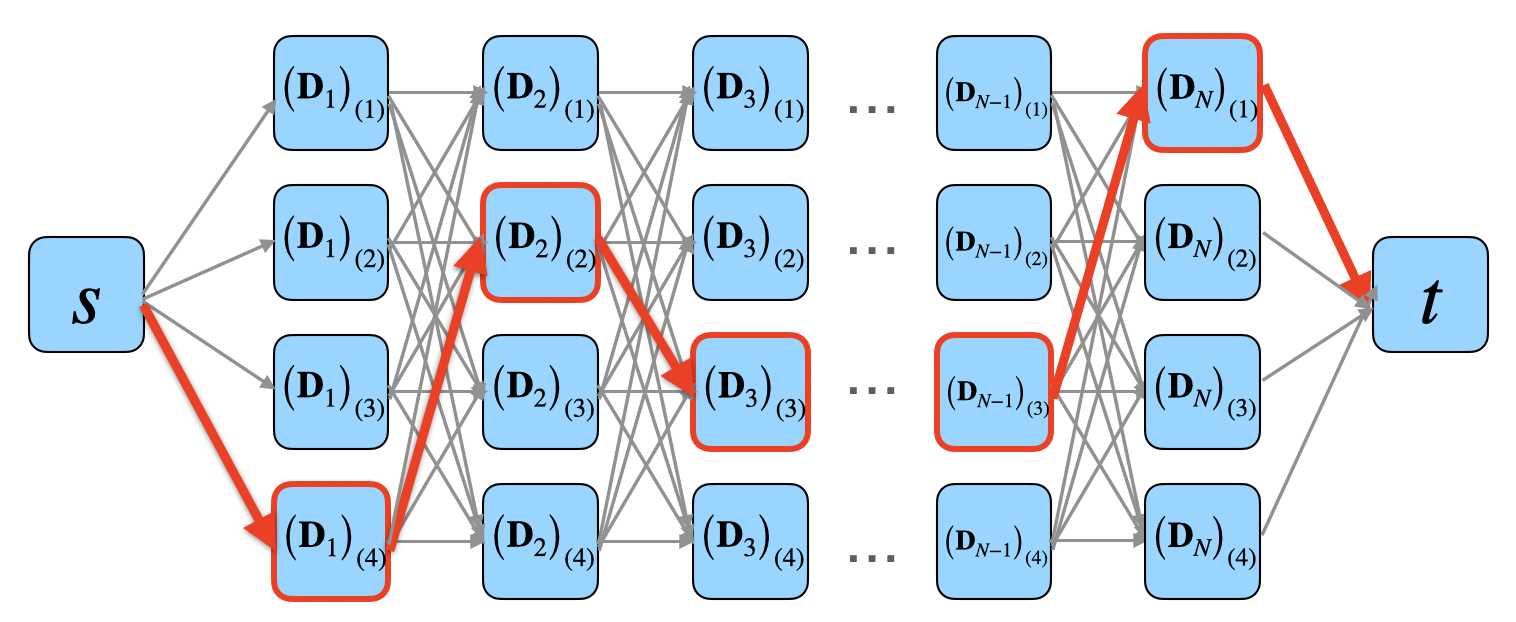}
        \caption{ A schematic diagram for \texttt{global\_matching} (Algorithm~\ref{alg: shortest path}). At each iteration, we find the shortest path from $s$ to $t$ (highlighted with red), estimate one atom of $\bD^{g*}$ using all passed vertices and remove the path (including the vertices) from $\mathcal{G}$. 
        }
        \label{fig:path}
\end{figure}
The above optimization aims to obtain the appropriate signed permutation matrices $\{\bPi_i\}_{i=1}^N$ that align the first $r^g$ atoms of the permuted dictionaries. In the ideal regime where $\bD_i^{(0)} = \bD_i^*, 1\leq i\leq N$, the optimal solution $\{\bPi^*_i\}_{i=1}^N$ yields a zero objective value and satisfies $\left(\bD_i^{(0)}\bPi^*_i\right)_{(1:r^g)} = \bD^{g*}$. However, there are two main challenges with the above optimization. First, it is a nonconvex, combinatorial problem over the discrete sets $\{\mathcal{P}(r_i)\}$. Second, the initial dictionaries may not coincide with their true counterparts. To address the first challenge, we show that the optimal solution to the optimization~\eqref{opt_matching} can be efficiently obtained by solving a series of shortest path problems defined over an auxiliary graph. To alleviate the second challenge, we show that our proposed algorithm is robust against possible errors in the initial dictionaries.

Consider a weighted $N$-layered {\it directed acyclic graph} (DAG) $\mathcal{G}$ with $r_i$ nodes in layer $i$ representing the $r_i$ atoms in $\bD_i$. We connect any node $a$ from layer $i$ to any node $b$ from layer $i+1$ with a directed edge with weight $w(a,b) = d_2\left(\left(\bD_i\right)_{a}, \left(\bD_{i+1}\right)_{b}\right)$. We add a \textit{source} node $s$ and connect it to all nodes in layer 1 with weight 0. Similarly, we include a \textit{terminal} node $t$ and connect all nodes in layer $N$ to $t$ with weight 0. A schematic construction of this graph is presented in Figure~\ref{fig:path}. Given the constructed graph, Algorithm~\ref{alg: shortest path} aims to solve~\eqref{opt_matching} by running $r^g$ rounds of the shortest path problem: at each round, the algorithm identifies the most aligned atoms in the initial dictionaries by obtaining the shortest path from $s$ to $t$. Then it removes the used nodes in the path for the next round. The correctness and robustness of the proposed algorithm are established in the next theorem.

\begin{algorithm}
	\caption{\texttt{global\_matching}}
	\label{alg: shortest path}
	\begin{algorithmic}[1]
		\STATE{{\bf Input:} $\left\{\bD^{(0)}_i\right\}_{i=1}^N$ and $r^g$.}
            \STATE{Construct the weighted $N$-layer DAG $\mathcal{G}$ described in Section~\ref{subsec: selection}.}
            \STATE Initialize $\{\mathrm{index}_i\}_{i=1}^N$ as empty sets and $\bD^{g,(0)}$ as an empty matrix.
            \FOR{$j=1,\dots,r^g$}
                \STATE{Find the shortest path $P = \left(s, (\bD^{(0)}_1)_{(\alpha_1)}, (\bD^{(0)}_1)_{(\alpha_2)},\cdots, (\bD^{(0)}_N)_{(\alpha_N)}, t\right)$.}
                \STATE{Add $\frac{1}{N}\sum_{i=1}^N\mathrm{sign}\bbracket{\left\langle(\bD^{(0)}_i)_{(\alpha_i)},(\bD^{(0)}_1)_{(\alpha_1)}\right\rangle}(\bD^{(0)}_i)_{(\alpha_i)}$ as a new column of $\bD^{g,(0)}$.}
                \STATE Add $\alpha_i$ to $\mathrm{index}_i$ for every $i=1,\dots,N$.
                \STATE{Remove $\mathcal{P}$ from $\mathcal{G}$}.
            \ENDFOR
            \STATE Set $\bD_i^{l,(0)}=(\bD^{(0)}_i)_{([r_i]\setminus\mathrm{index}_i)}$ for every $i=1,\dots,N$.
		\RETURN{$\left(\bD^{g,(0)}, \left\{\bD_i^{l,(0)}\right\}_{i=1}^N\right)$.}
	\end{algorithmic}
\end{algorithm}

\begin{theorem}[Correctness and robustness of \texttt{global\_matching}]\label{theorem: shortest path}
Suppose $\{\bD^*_i\}_{i=1}^N$ are $\mu$-incoherent (Assumption~\ref{def: incoherency}) and $\beta$-identifiable (Assumption~\ref{def: identifiablity}). Suppose the initial dictionaries $\left\{\bD^{(0)}_i\right\}_{i=1}^N$ satisfy $d_{1,2}\left(\bD^{(0)}_i,\bD^*_i\right)\le \epsilon_i$ with $4\sum_{i=1}^N \epsilon_i \le \min\left\{\sqrt{2-2\frac{\mu}{\sqrt{d}}},\beta\right\}$. Then, the output of Algorithm~\ref{alg: shortest path} satisfies:
\begin{align}
    d_{1,2}\left(\bD^{g,(0)} , \bD^{g*}\right) \le \frac{1}{N}\sum_{i=1}^N \epsilon_i,\qquad\text{and}\qquad 
    d_{1,2}\left(\bD^{l,(0)}_i , \bD^{l*}_i\right)\le \epsilon_i, \quad\mathrm{for}\quad 1\le i\le N.
\end{align}
\end{theorem}

According to the above theorem, \texttt{global\_matching} can robustly separate the clients' initial dictionaries into global and local parts, provided that the aggregated error in the initial dictionaries is below a threshold. Specific initialization schemes that can satisfy the condition of Theorem~\ref{theorem: shortest path} include Algorithm 1 from \cite{agarwal2013exact} and Algorithm 3 from \cite{arora2015simple}. We also remark that since the constructed graph is a DAG, the shortest path problem can be solved in time linear in the number of edges, which is $\mathcal{O}\left(r_1+r_N+\sum_{i=1}^{N-1}r_ir_{i+1}\right)$, via a
simple labeling algorithm (see, e.g.,~\citep[Chapter 4.4]{ahuja1988network}). Since we need to solve the shortest path problem $r^g$ times, this brings the computational complexity of Algorithm~\ref{alg: shortest path} to $\mathcal{O}(r^gNr_{\max}^2)$, where $r_{\max} = \max_i r_i$.

Given the initial local and global dictionaries, the clients progressively refine their estimates by applying $T$ rounds of \texttt{local\_update} (Algorithm~\ref{alg: local atom separation}). 
At a high level, each client runs a single iteration of a \textit{linearly convergent algorithm} $\mathcal{A}_i$ (see Definition~\ref{def_lin_alg}), followed by an alignment step that determines the global atoms of the updated dictionary using $\bD^{g,(t)}$ as a "reference point". Notably, our implementation of \texttt{local\_update} is adaptive to different DL algorithms. This flexibility is indeed intentional to provide a versatile meta-algorithm for clients with different DL algorithms.

\begin{algorithm}
    \caption{\texttt{local\_update}}
    \label{alg: local atom separation}
    \begin{algorithmic}[1]
        \STATE{{\bf Input:} $\bD^{(t)}_i = \left[\bD^{g,(t)}\quad \bD^{l,(t)}_i\right], \bY_i$}
        \STATE{$\bD^{(t+1)}_i = \mathcal{A}_i\left(\bY_i, \bD^{(t)}_i\right)$} $\hfill\color{cyan} \texttt{// One iteration of a linearly-convergent algorithm.}$
        \STATE{Initialize $\mathcal{S}$ as an empty set and $\bP \in\R^{r^g\times r^g}$ as an all-zero matrix.}
        \FOR{$j=1,...,r^g$}
            \STATE{Find $k^*=\arg\min_k d_2\bbracket{\bbracket{\bD^{g,(t)}}_{(j)},\bbracket{\bD^{(t+1)}_i}_{(k)}}$}.
            \STATE{Append $k^*$ to $\mathcal{S}$.}
            \STATE{Set $(i,i)$-th entry of $\bP$ to $\mathrm{sign}\bbracket{\left\langle\bbracket{\bD^{g,(t)}}_{(j)},\bbracket{\bD^{(t+1)}_i}_{(k^*)}\right\rangle}$.}
        \ENDFOR
        \STATE{{\bf Output:} $\bD^{g,(t+1)}_i = \bbracket{\bD^{(t+1)}_i}_{(\mathcal{S})}\bP$ and  $\bD^{l,(t+1)}_i = \bbracket{\bD^{(t+1)}_i}_{([r_i]\setminus\mathcal{S})}$}.
    \end{algorithmic}
\end{algorithm}

\section{Theoretical Guarantees}\label{sec: theory}
In this section, we show that our proposed meta-algorithm provably solves $\mathsf{PerDL}$ under suitable initialization, identifiability, and algorithmic conditions. To achieve this goal, we first present the definition of a linearly-convergent DL algorithm.
\begin{definition}\label{def_lin_alg}
    Given a generative model $\rmY=\rmD^*\rmX^*$, a DL algorithm $\mathcal{A}$ is called { ${(\delta,\rho,\psi)}$-linearly convergent} for some parameters $\delta, \psi>0$ and $0<\rho<1$ if, for any $\bD\in \R^{d\times r}$ such that $d_{1,2}(\bD,\bD^*)\le \delta$, the output of one iteration $\bD^+=\mathcal{A}(\bD,\rmY)$, satisfies 
    \begin{align}
        d_{2,(j)}\bbracket{\bD^+,\bD^*}\le \rho d_{2,(j)}\bbracket{\bD,\bD^*}+\psi,\quad\forall  1\le j\le r.
    \end{align}
\end{definition}

One notable linearly convergent algorithm is introduced by~\cite[Algorithm 5]{arora2015simple}; we will discuss this algorithm in more detail in the appendix. Assuming all clients are equipped with linearly convergent algorithms, our next theorem establishes the convergence of $\mathsf{PerMA}$.

\begin{theorem}[Convergence of $\mathsf{PerMA}$]\label{theorem: convergence PerMA}
Suppose $\{\bD^*_i\}_{i=1}^N$ are $\mu$-incoherent (Assumption~\ref{def: incoherency}) and $\beta$-identifiable (Assumption~\ref{def: identifiablity}). Suppose, for every client $i$, the DL algorithm $\mathcal{A}_i$ used in \texttt{local\_update} (Algorithm~\ref{alg: local atom separation}) is $(\delta_i,\rho_i,\psi_i)$-linearly convergent with $4\sum_{i=1}^N \delta_i \le \min\left\{\sqrt{2-2\frac{\mu}{\sqrt{d}}},\beta\right\}$. Assume the initial dictionaries $\{\bD^{(0)}_i\}_{i=1}^N$ satisfy:
\begin{align}
    d_{1,2}\bbracket{\frac{1}{N}\sum_{i=1}^N\bD^{g,(0)}_i,\bD^{g*}}\le \min_{1\le i \le N} \delta_i, \quad d_{1,2}\bbracket{\bD^{l,(0)}_i,\bD^{l*}_i}\le \delta_i, \quad \text{for}\quad i=1,\dots, N.
\end{align}
Then, for every $t\geq 0$, $\mathsf{PerMA}$ (Algorithm~\ref{alg: general alg}) satisfies
\begin{align}
        d_{1,2}\left(\bD^{g,(t+1)},\bD^{g*}\right) &\le \bbracket{\frac{1}{N}\sum_{i=1}^N\rho_i }d_{1,2}\left(\bD^{g,(t)},\bD^{g*}\right)+{\frac{1}{N}\sum_{i=1}^N\psi_i},
        \label{eq: global convergence}\\
        d_{1,2}\left(\bD^{l,(t+1)}_i,\bD^{l*}_i\right) &\le \rho_i d_{1,2}\left(\bD^{l,(t)},\bD^{l*}_i\right)+{\psi_i},\quad\mathrm{for}\quad 1\le i\le N.\label{eq: local convergence}
    \end{align}
\end{theorem}
The above theorem sheds light on a number of key benefits of $\mathsf{PerMA}$:

\paragraph{Relaxed initial condition for weak clients.} Our algorithm relaxes the initial condition on the global dictionaries. In particular, it only requires the average of the initial global dictionaries to be close to the true global dictionary. Consequently, it enjoys a provable convergence guarantee even if some of the clients do not provide a high-quality initial dictionary to the server.

\paragraph{Improved convergence rate for slow clients.} During the course of the algorithm, the global dictionary error decreases at an average rate of $\frac{1}{N}\sum_{i=1}^N\rho_i$, improving upon the convergence rate of the slow clients. 

\paragraph{Improved statistical error for weak clients.} A linearly convergent DL algorithm $\mathcal{A}_i$ stops making progress upon reaching a neighborhood around the true dictionary $\bD_i^*$ with radius $O(\psi_i)$. This type of guarantee is common among DL algorithms \citep{arora2015simple,liang2022simple} and often corresponds to their statistical error. In light of this, $\mathsf{PerMA}$ improves the performance of weak clients (i.e., clients with weak statistical guarantees) by borrowing strength from strong ones.

\section{Numerical Experiments}\label{sec: numerical}
In this section, we showcase the effectiveness of Algorithm~\ref{alg: general alg} using synthetic and real data. All experiments are performed on a MacBook Pro 2021 with the Apple M1 Pro chip and 16GB unified memory for a serial implementation in MATLAB 2022a. Due to limited space, we will only provide high-level motivation and implication of our experiments and defer implementation details to the appendix. 

\begin{figure}
        \centering
        \includegraphics[width=0.7\textwidth]{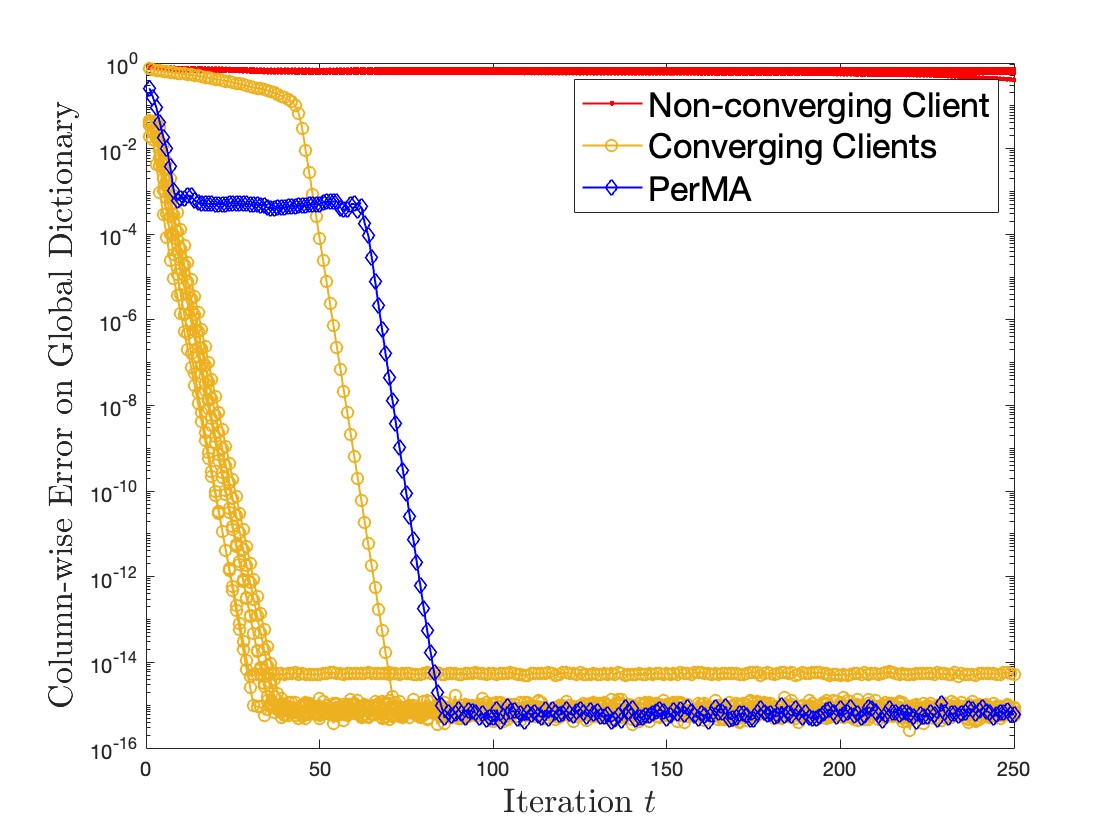}
        \caption{ $\mathsf{PerMA}$ improves the accuracy of the recovered global dictionary for \textit{all} clients, even if some (three out of ten) are weak learners.}
        \label{fig:synthetic}
\end{figure}

\subsection{Synthetic Dataset}\label{sec: synthetic dataset}
In this section, we validate our theoretical results on a synthetic dataset. We consider ten clients, each with a dataset generated according to the model~\ref{model: PerDL}. The details of our construction are presented in the appendix. Specifically, we compare the performances of two strategies: (1) {\it independent strategy}, where each client solves \ref{objective DL} without any collaboration, and (2) {\it collaborative strategy}, where clients collaboratively learn the ground truth dictionaries by solving $\mathsf{PerDL}$ via the proposed meta-algorithm $\mathsf{PerMA}$. We initialize both strategies using the same $\{\bD^{(0)}_i\}_{i=1}^{N}$. The initial dictionaries are obtained via a warm-up method proposed in~\cite[Algorithm 4]{liang2022simple}. For a fair comparison between independent and collaborative strategies, we use the same DL algorithm (\cite[Algorithm 1]{liang2022simple}) for different clients. Note that in the independent strategy, the clients cannot separate global from local dictionaries. Nonetheless, to evaluate their performance, we collect the atoms that best align with the true global dictionary $\bD^{g*}$ and treat them as the estimated global dictionaries. As can be seen in Figure~\ref{fig:synthetic}, three out of ten clients are weak learners and fail to recover the global dictionary with desirable accuracy. On the contrary, in the collaborative strategy, all clients recover the same global dictionary almost exactly.

\subsection{Training with Imbalanced Data}\label{sec: imbalanced data}
In this section, we showcase the application of $\mathsf{PerDL}$ in training with imbalanced datasets. We consider an image reconstruction task on MNIST dataset. This dataset corresponds to a set of handwritten digits (see the first row of Figure~\ref{fig:mnist}). The goal is to recover a \textit{single} concise global dictionary that can be used to reconstruct the original handwritten digits as accurately as possible. 
In particular, we study a setting where the clients have \textit{imbalanced label distributions}. Indeed, data imbalance can drastically bias the performance of the trained model in favor of the majority groups, while hurting its performance on the minority groups~\citep{leevy2018survey, thabtah2020data}. Here, we consider a setting where the clients have highly imbalanced datasets, where $90\%$ of their samples have the same label. More specifically,  for client $i$, we assume that $90\%$ of the samples correspond to the handwritten digit ``$i$'', with the remaining $10\%$ corresponding to other digits. The second row of Figure~\ref{fig:mnist} shows the effect of data imbalance on the performance of the recovered dictionary on a single client, when the clients do not collaborate. The last row of Figure~\ref{fig:mnist} shows the improved performance of the recovered dictionary via $\mathsf{PerDL}$ on the same client. Our experiment clearly shows the ability of $\mathsf{PerDL}$ to effectively address the data imbalance issue by combining the strengths of different clients.

\begin{figure}
    \centering
        \includegraphics[width=\textwidth]{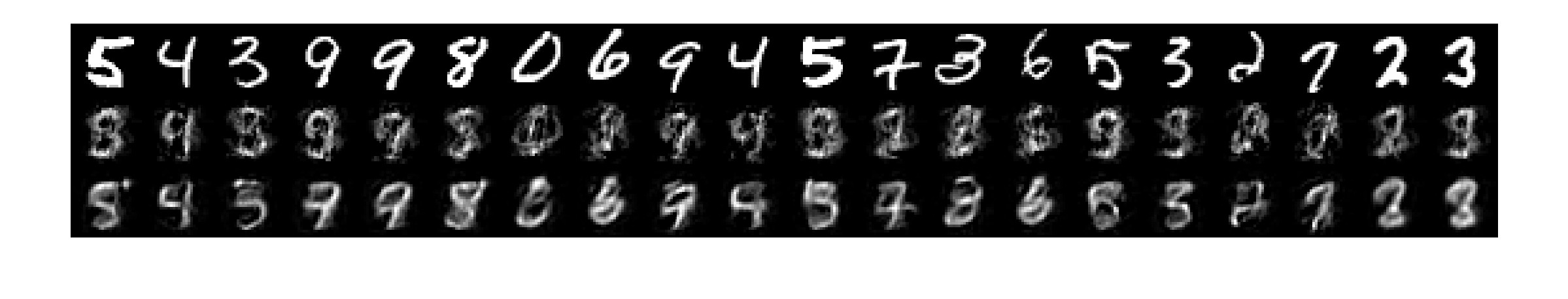}
        \caption{ $\mathsf{PerMA}$ improves training with imbalanced datasets. We consider the image reconstruction task on the imbalanced MNIST dataset using only five atoms from a learned global dictionary. The first row corresponds to the original images. The second row is based on the dictionary learned on a single client with an imbalanced dataset. The third row shows the improved performance of the learned dictionary using our proposed method on the same client.}
        \label{fig:mnist}
\end{figure}
\subsection{Surveillance Video Dataset}\label{sec: surveillance dataset}
As a proof of concept, we consider a video surveillance task, where the goal is to separate the background from moving objects. Our data is collected from~\cite{cuevas2016labeled} (see the first column of Figure~\ref{fig:car1}). As these frames are taken from one surveillance camera, they share the same background corresponding to the global features we aim to extract. The frames also exhibit heterogeneity as moving objects therein are different from the background. This problem can indeed be modeled as an instance of $\mathsf{PerDL}$, where each video frame can be assigned to a ``client'', with the global dictionary capturing the background and local dictionaries modeling the moving objects. We solve $\mathsf{PerDL}$ by applying $\mathsf{PerMA}$ to obtain a global dictionary and several local dictionaries for this dataset. Figure~\ref{fig:car1} shows the reconstructed background and moving objects via the recovered global and local dictionaries. Our results clearly show the ability of our proposed framework to separate global and local features. \footnote{We note that moving object detection in video frames has been extensively studied in the literature and typically solved very accurately via different representation learning methods (such as robust PCA and neural network modeling); see~\citep{yazdi2018new} for a recent survey. Here, we use this case study as a proof of concept to illustrate the versatility of $\mathsf{PerMA}$ in handling heterogeneity, even in settings where the data is not physically distributed among clients.}

\begin{figure}
    \centering
        \includegraphics[width=1\textwidth]{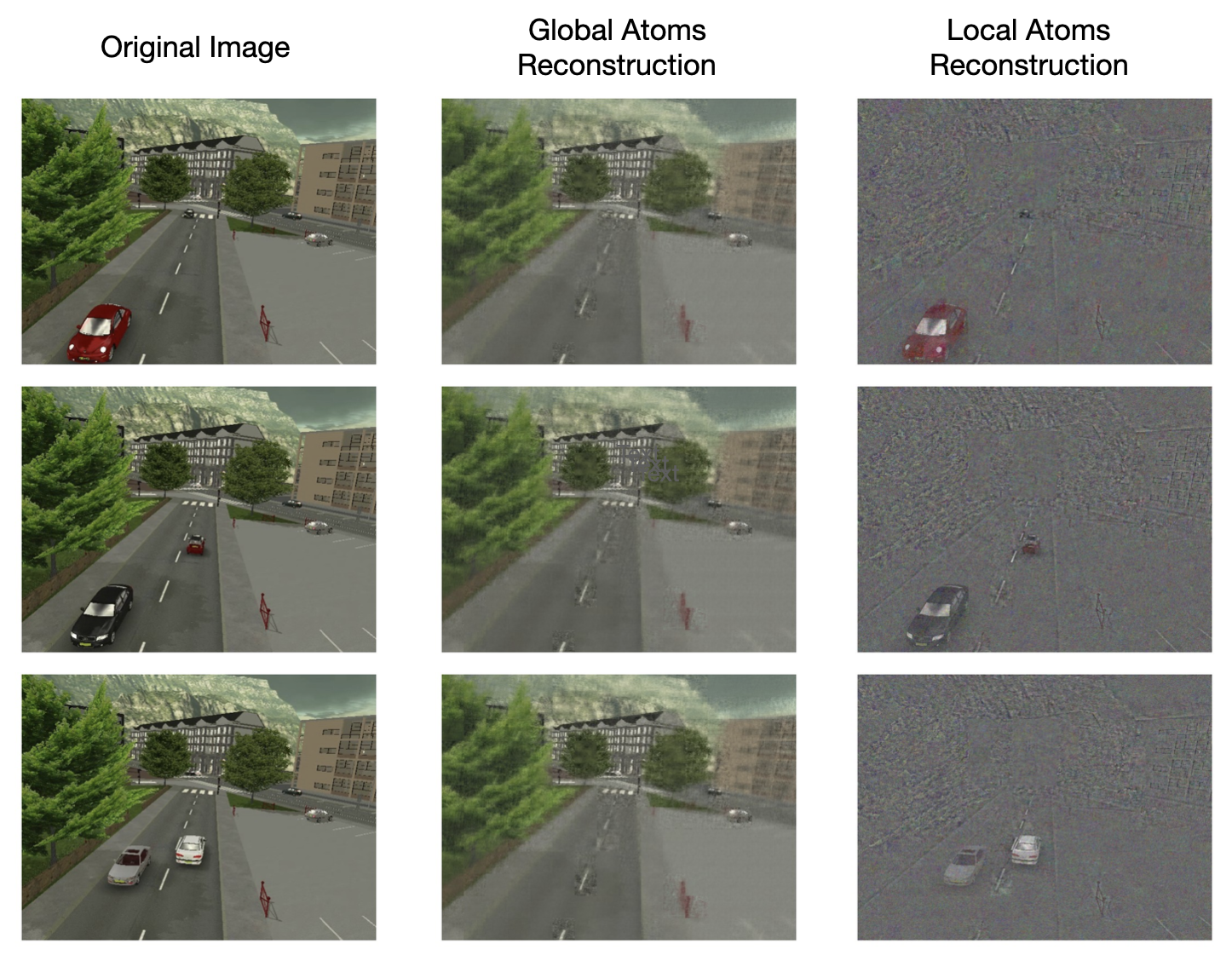}
        \caption{$\mathsf{PerMA}$ effectively separates the background from moving objects in video frames. 
        Here we reconstruct the surveillance video frames using global and local dictionaries learned by $\mathsf{PerMA}$. We reconstruct the frames using only $50$ atoms from the combined dictionaries. }
        \label{fig:car1}
\end{figure}

\section{Conclusion and Future Directions}
In this work, we introduce the problem of \textit{Personalized Dictionary Learning} (PerDL), where the goal is to learn global and local dictionaries for heterogeneous datasets. One promising direction for future research is to relax the requirement on the linear convergence of the DL algorithm (Definition~\ref{def_lin_alg}). Moreover, it would be worthwhile to investigate whether the convergence and statistical guarantees of our proposed algorithm can be improved beyond Theorem~\ref{theorem: convergence PerMA}.

\section*{Acknowledgements}
SF is supported, in part, by NSF Award DMS-2152776, ONR Award N00014-22-1-2127, and MICDE Catalyst Grant. RA is supported, in part, by NSF CAREER Award CMMI-2144147.

\bibliography{ref}
\bibliographystyle{apalike}
\clearpage
\appendix

\section{Further Details on the Experiments}
In this section, we provide further details on the numerical experiments reported in Section~\ref{sec: numerical}.
\subsection{Details of Section~\ref{sec: synthetic dataset}}
In Section~\ref{sec: synthetic dataset}, we generate the synthetic datasets according to Model~\ref{model: PerDL} with $N=10$, $d=6$, $r_i=6$, and $n_i=200$ for each $1\le i\le 10$. Each $\bD^*_i$ is an orthogonal matrix with the first $r^g=3$ columns shared with every other client and the last $r^l_i=3$ columns unique to themselves. Each $\bX^*_i$ is first generated from a Gaussian-Bernoulli distribution where each entry is non-zero with a probability $0.2$. Then, $\bX^*_i$ is further truncated, where all the entries $\bbracket{\bX^*_i}_{(j,k)}$ with $|\bbracket{\bX^*_i}_{(j,k)}|< 0.3$ are replaced by $\bbracket{\bX^*_i}_{(j,k)} = 0.3\times\mathrm{sign}(\bbracket{\bX^*_i}_{(j,k)})$. \par

We use the orthogonal DL algorithm (Algorithm~\ref{alg: ortho offline}) introduced in~\cite[Algorithm 1]{liang2022simple} as the local DL algorithm for each client. This algorithm is simple to implement and comes equipped with a strong convergence guarantee (see~\cite[Theorem 1]{liang2022simple}).
Here $\mathrm{HT}_\zeta(\cdot)$ denotes the hard-thresholding operator at level $\zeta$, which is defined as:
\[
    \left(\mathrm{HT}_\zeta(\bA)\right)_{(i,j)}=
    \begin{cases}
    \bA_{(i,j)}&\text{if}\quad |\bA_{(i,j)}|\ge \zeta,\\
    0&\text{if}\quad |\bA_{(i,j)}|<\zeta.\\
    \end{cases}
\]
Specifically, we use $\zeta = 0.15$ for the experiments in Section~\ref{sec: synthetic dataset}. $\mathrm{Polar}(\cdot)$ denotes the polar decomposition operater, which is defined as $\mathrm{Polar}(\bA)=\bU_{\bA} \bV_{\bA}^\top$, where $\bU_{\bA} \bSigma_{\bA} \bV_{\bA}^\top$ is the Singular Value Decomposition (SVD) of $\bA$.
\begin{algorithm}[H]
	\caption{Alternating minimization for orthogonal dictionary learning (\cite{liang2022simple})}
	\label{alg: ortho offline}
	\begin{algorithmic}[1]
		\STATE{{\bf Input:} $\bY_i$, $\bD^{(t)}_i$}
		\STATE{Set $\bX^{(t)}_i = \mathrm{HT}_{\zeta}\left(\bD^{(t)_i\top}\bY_i\right)$}\label{step: hard thresholding}
		\STATE{Set $\bD^{(t+1)}_i = \mathrm{Polar}\left(\bY_i \bX^{(t)\top}_i\right)$}
		\RETURN{$\bD^{(t+1)}_i$}
	\end{algorithmic}
\end{algorithm}
For a fair comparison, we initialize both strategies using the same $\{\bD^{(0)}_i\}_{i=1}^{N}$, which is obtained by iteratively calling Algorithm~\ref{alg: ortho offline} with a random initial dictionary and  shrinking thresholds $\zeta$. For a detailed discussion on such an initialization scheme we refer the reader to \cite{liang2022simple}. 
\subsection{Details of Section~\ref{sec: imbalanced data}}
In section~\ref{sec: imbalanced data}, we aim to learn a dictionary with imbalanced data collected from MNIST dataset \citep{lecun2010mnist}. Specifically, we consider $N=10$ clients, each with $500$ handwritten images. Each image is comprised of $28\times 28$ pixels. Instead of randomly assigning images, we construct dataset $i$ such that it contains $450$ images of digit $i$ and $50$ images of other digits. Here client $10$ corresponds to digit $0$. After vectorizing each image into a $784\times 1$ one-dimension signal, our imbalanced dataset contains $10$ matrices $\bY_i\in\R^{ 784\times 500}, i=1,\dots, 10$.\par
We first use Algorithm~\ref{alg: ortho offline} to learn an orthogonal dictionary for each client, using their own imbalanced dataset.  For client $i$, given the output of Algorithm~\ref{alg: ortho offline} after $T$ iterations $\bD^{(T)}_i$, we reconstruct a new signal $\by$ using the top $k$ atoms according to the following steps: first, we solve a \textit{sparse coding} problem to find the sparse code $\bx$ such that $\by\approx \bD^{(T)}_i\bx$. This can be achieved by Step~\ref{step: hard thresholding} in Algorithm~\ref{alg: ortho offline}. Second, we find the top $k$ entries in $\bx$ that have the largest magnitude: $\bx_{(\alpha_1,1)}$, $\bx_{(\alpha_2,1)},\cdots,\bx_{(\alpha_k,1)}$. Finally, we calculate the reconstructed signal $\Tilde{\by}$ as
\[
\Tilde{\by} = \sum_{j=1}^k \bx_{(\alpha_h,1)}\bbracket{\bD^{(T)}_i}_{\alpha_h}.
\]
The second row of Figure~\ref{fig:mnist} is generated by the above procedure with $k=5$ using the dictionary learned by Client $1$. 
The third row of Figure~\ref{fig:mnist} corresponds to the reconstructed images using the output of $\mathsf{PerMA}$.

\subsection{Details of Section~\ref{sec: surveillance dataset}}
Our considered dataset in section~\ref{sec: surveillance dataset} contains 62 frames, each of which is a $480\times 640\times 3$ RGB image. We consider each frame as one client ($N=62$). After dividing each frame into $40\times 40$ patches, we obtain each data matrix $\bY_i\in\R^{576\times1600}$. Then we apply $\mathsf{PerMA}$ to $\{\bY_i\}_{i=1}^{62}$ with $r_i= 576$ for all $i$ and $r^g = 30$. Consider $\bD^{(T)}_i=\left[\bD^{g,(T)}\quad \bD^{l,(T)}_i\right]$, which is the output of $\mathsf{PerMA}$ for client $i$. We reconstruct each $\bY_i$ using the procedure described in the previous section with $k=50$. Specifically, we separate the contribution of $\bD^{g,(T)}$ from $\bD^{l,(T)}_i$. Consider the reconstructed matrix $\Tilde{Y}_i$ as 
\[
    \Tilde{\bY}_i = \begin{bmatrix}
        \bD^{g,(T)}&\bD^{l,(T)}_i
    \end{bmatrix}
    \begin{bmatrix}
        \bX^{g}_i\\\bX^{l}_i
    \end{bmatrix}
    = \underbrace{\bD^{g,(T)}\bX^{g}_i}_{\Tilde{\bY}_i^g} + \underbrace{\bD^{l,(T)}_i\bX^{l}_i}_{\Tilde{\bY}_i^l}
\]
The second column and the third column of Figure~\ref{fig:car1} correspond to reconstructed results of $\Tilde{\bY}_i^g$ and $\Tilde{\bY}_i^l$ respectively. We can see clear separation of the background (which is shared among all frames) from the moving objects (which is unique to each frame).\par
One notable difference between this experiment and the previous one is in the choice of the DL algorithm $\mathcal{A}_i$. To provide more flexibility, we relax the orthogonality condition for the dictionary. Therefore, we use the alternating minimization algorithm introduced in~\cite{arora2015simple} for each client (see Algorithm~\ref{alg: over offline}). The main difference between this algorithm and Algorithm~\ref{alg: ortho offline} is that the polar decomposition step in Algorithm~\ref{alg: ortho offline} is replaced by a single iteration of the gradient descent applied to the loss function $\mathcal{L}(\bD,\bX) = \|\bD\bX-\bY\|^2_F$.
\par 
\begin{algorithm}[H]
	\caption{Alternating minimization for general dictionary learning (\cite{arora2015simple})}
	\label{alg: over offline}
	\begin{algorithmic}[1]
		\STATE{{\bf Input:} $\bY_i$, $\bD^{(t)}_i$}
		\STATE{Set $\bX^{(t)}_i = \mathrm{HT}_{\zeta}\left(\bD^{(t)\top}_i\bY_i\right)$}\label{step: hard thresholding 2}
		\STATE{Set $\bD^{(t+1)}_i = \bD^{(t)}_i-2\eta \bbracket{\bD^{(t)}_i\bX^{(t)}_i-\bY_i}\bX^{(t)\top}_i$}
		\RETURN{$\bD^{(t+1)}_i$}
	\end{algorithmic}
\end{algorithm}
Even with the computational saving brought up by Algorithm~\ref{alg: over offline}, the runtime significantly slows down for $\mathsf{PerMA}$ due to large $N$, $d$, and $p$. Here we report a practical trick to speed up $\mathsf{PerMA}$, which is a local refinement procedure (Algorithm~\ref{alg: local dic refinement}) added immediately before \texttt{local\_update} (Step 10 of Algorithm~\ref{alg: general alg}). At a high level, \texttt{local\_dictionary\_refinement} first finds the local residual data matrix $\bY^{l}_i$ by removing the contribution of the global dictionary. Then it iteratively refines the local dictionary with respect to $\bY^{l}_i$. We observed that \texttt{local\_dictionary\_refinement} significantly improves the local reconstruction quality. We leave its theoretical analysis as a possible direction for future work.
\begin{algorithm}[H]
	\caption{\texttt{local\_dictionary\_refinement}}
	\label{alg: local dic refinement}
	\begin{algorithmic}[1]
	\STATE{{\bf Input:} $\bD^{(t)}_i = \left[\bD^{g,(t)}\quad \bD^{l,(t)}_i\right], \bY_i$}
	\STATE{Find $\begin{bmatrix}
        \bX^{g}_i\\\bX^{l}_i
    \end{bmatrix}$ such that $\bY_i\approx \left[\bD^{g,(t)}\quad \bD^{l,(t)}_i\right]\begin{bmatrix}
        \bX^{g}_i\\\bX^{l}_i
    \end{bmatrix}$\\$\hfill\color{cyan} \texttt{// Solving a sparse coding problem}$}
	\STATE{Set $\bY^{l}_i = \bY_i - \bD^{g,(t)}\bX^{g}_i$}
        \STATE{Set $\bD^{\mathrm{refine},(0)}_i = \bD^{l,(t)}_i$.}
        \FOR{$\tau=0,1,...,T^{\mathrm{refine}}-1$}
            \STATE{Set $\bD^{\mathrm{refine},(\tau+1)}_i = \mathcal{A}_i\bbracket{\bY^{l}_i,\bD^{\mathrm{refine},(\tau)}_i}$ $\hfill\color{cyan} \texttt{// Improving local dictionary}$}
        \ENDFOR
	\RETURN{$\bD^{\mathrm{refine},(T^{\mathrm{refine}})}_i$ as refined $\bD^{l,(t)}_i$}
	\end{algorithmic}
\end{algorithm}

\section{Further Discussion on Linearly Convergent Algorithms}
In this section, we discuss a linearly convergent DL algorithm that satisfies the conditions of our Theorem~\ref{theorem: convergence PerMA}. In particular, the next theorem is adapted from \cite[Theorem 12]{arora2015simple} and shows that a modified variant of Algorithm~\ref{alg: over offline} introduced in \cite[Algorithm 5]{arora2015simple} is indeed linearly-convergent.

\begin{theorem}[Linear convergence of Algorithm 5 in \cite{arora2015simple}]
\label{theorem: convergence without collab} Suppose that the data matrix satisfies 
$\bY = \bD^*\bX^*$, where $\bD^*$ is an  $\mu$-incoherent dictionary and the sparse code $\bX^*$ satisfies the generative model introduced in Section 1.2 and Section 4.1 of \cite{arora2015simple}. For any initial dictionary $\|\bD^{(0)}\|_2\leq 1$, Algorithm 5 in \cite{arora2015simple} is $(\delta,\rho,\psi)$-linearly convergent with $\delta = O(1/\log d)$, $\rho\in(1/2,1)$, and $\psi = O(d^{-\omega(1)})$.
\end{theorem}

Algorithm 5 in \cite{arora2015simple} is a refinement of Algorithm~\ref{alg: over offline}, where the error is further reduced by projecting out the components along the column currently being updated. For brevity, we do not discuss the exact implementation of the algorithm; an interested reader may refer to \cite{arora2015simple} for more details. Indeed, we have observed in our experiments that the additional projection step does not provide a significant benefit over Algorithm~\ref{alg: over offline}.

\section{Proofs of Theorems}
\subsection{Proof of Theorem~\ref{theorem: shortest path}}
To begin with, we establish a triangular inequality for $d_{1,2}(\cdot,\cdot)$, which will be important in our subsequent arguments:
\begin{lemma}[Triangular inequality for $d_{1,2}(\cdot,\cdot)$]\label{lemma: triangle for tilde d}
    For any dictionary $\bD_1$, $\bD_2$, $\bD_3\in\R^{d\times r}$, we have
    \begin{equation}
        d_{1,2}\bbracket{\D_1,\D_2}\le d_{1,2}\bbracket{\D_1,\D_3}+d_{1,2}\bbracket{\D_3,\D_2}
    \end{equation}
\end{lemma}
\begin{proof}
    Suppose $\bPi_{1,3}$ and $\bPi_{3,2}$ satisfy $d_{1,2}\bbracket{\D_1,\D_3}=\|\bD_1\bPi_{1,3}-\D_3\|_{1,2}$ and $d_{1,2}\bbracket{\D_3,\D_2}=\|\bD_3-\bD_2\bPi_{3,2}\|_{1,2}$. Then we have
    \begin{equation}
        \begin{aligned}
        d_{1,2}\bbracket{\D_1,\D_3} + d_{1,2}\bbracket{\D_3,\D_2}&=
        \|\bD_1\bPi_{1,3}-\D_3\|_{1,2} + \|\bD_3-\bD_2\bPi_{3,2}\|_{1,2}\\
        &\ge \|\bD_1\bPi_{1,3}-\bD_2\bPi_{3,2}\|_{1,2}\\
        &\ge d_{1,2}\bbracket{\D_1,\D_2}.
    \end{aligned}
    \end{equation}
\end{proof}
Given how the directed graph $\mathcal{G}$ is constructed and modified, any directed path from $s$ to $t$ will be of the form $\mathcal{P} = s\rightarrow (\D^{(0)}_1)_{\alpha(1)}\rightarrow (\D^{(0)}_2)_{\alpha(2)}\rightarrow\cdots\rightarrow (\D^{(0)}_N)_{\alpha(N)}\rightarrow t$. Specifically, each layer (or client) contributes exactly one node (or atom), and the path is determined by $\alpha(\cdot):[N]\rightarrow[r]$. 
Recall that $\bD^*_i = \begin{bmatrix}
    \bD^{g*} & \bD^{l*}_i
\end{bmatrix}$ for every $1\leq i\leq N$. Assume, without loss of generality, that for every client $1\le i \le N$, 
\begin{equation}\label{assumption:aligned}
    \bI_{r_i\times r_i}=\arg\min_{\bPi\in\mathcal{P}(r_i)}\left\|\bD^*_i\bPi-\bD^{(0)}_i\right\|_{1,2}.
\end{equation}
In other words, the first $r^g$ atoms in the initial dictionaries $\{D_i^{(0)}\}_{i=1}^N$ are aligned with the global dictionary. Now consider the special path $\mathcal{P}^*_j$ for $1\le j \le r^g $ defined as
\begin{equation}
    \mathcal{P}^*_j = s\rightarrow (\D^{(0)}_1)_j\rightarrow (\D^{(0)}_2)_j\rightarrow\cdots\rightarrow (\D^{(0)}_N)_j\rightarrow t.
\end{equation}
To prove that Algorithm~\ref{alg: shortest path} correctly selects and aligns global atoms from clients, it suffices to show that $\{\mathcal{P}^*_j\}_{j=1}^{r^g}$ are the top-$r^g$ shortest paths from $s$ to $t$ in $\mathcal{G}$. The length of the path $\mathcal{P}^*_j$ can be bounded as
\begin{equation}
\begin{aligned}
    \mathcal{L}\bbracket{\mathcal{P}^*_j}
    &= \sum_{i=1}^{N-1}d_2\left((\D^{(0)}_i)_j,(\D^{(0)}_{i+1})_j\right)\\
    &= \sum_{i=1}^{N-1}\min\left\{\|(\D^{(0)}_i)_j-(\D^{(0)}_{i+1})_j\|_2, \|(\D^{(0)}_i)_j+(\D^{(0)}_{i+1})_j\|_2\right\}\\
    &\le \sum_{i=1}^{N-1}\|(\D^{(0)}_i)_j - (\D^{(0)}_{i+1})_j\|_2\\
    &\le \sum_{i=1}^{N-1}\|(\D^{(0)}_i)_j - (\D^{g*})_j\|_2+\| (\D^{(0)}_{i+1})_j  - (\D^{g*})_j\|_2\\
    &\le \sum_{i=1}^{N-1} (\epsilon_i + \epsilon_{i+1})\\
    &\le 2 \sum_{i=1}^{N} \epsilon_i.
\end{aligned}
\end{equation}
 We move on to prove that all the other paths from $s$ to $t$ will have a distance longer than $2 \sum_{i=1}^{N} \epsilon_i$. Consider a general directed path $\mathcal{P} = s\rightarrow (\D^{(0)}_1)_{\alpha(1)}\rightarrow (\D^{(0)}_2)_{\alpha(2)}\rightarrow\cdots\rightarrow (\D^{(0)}_N)_{\alpha(N)}\rightarrow t$ that is not in $\{\mathcal{P}^*_j\}_{j=1}^{r^g}$. Based on whether or not $\mathcal{P}$ contains atoms that align with the true global ground atoms, there are two situations:\par
{\bf Case 1:} Suppose there exists $1\le i \le N$ such that $\alpha(i)\le r^g$. Given Model~\ref{model: PerDL} and the assumed equality~\eqref{assumption:aligned}, we know that for layer $i$, $\mathcal{P}$ contains a global atom. Since $\mathcal{P}$ is not in $\{\mathcal{P}^*_j\}_{j=1}^{r^g}$, there must exist $k\ne i$ such that $\alpha(k)\ne \alpha(i)$. As a result, we have
\begin{equation}
\begin{aligned}
    \mathcal{L}(\mathcal{P})
    &\overset{(a)}{\ge} d_{1,2}\bbracket{(\D^{(0)}_i)_{\alpha(i)},(\D^{(0)}_k)_{\alpha(k)}} \\
    &\overset{(b)}{\ge} \min\left\{\|(\D^{*}_i)_{\alpha(i)}-(\D^{*}_k)_{\alpha(k)}\|_2, (\D^{*}_i)_{\alpha(i)}+(\D^{*}_k)_{\alpha(k)}\|_2\right\}\\
    &\ \ \ \ \ - \|(\D^{*}_i)_{\alpha(i)}-(\D^{(0)}_i)_{\alpha(i)}\|_2 -  \|(\D^{*}_k)_{\alpha(k)}-(\D^{(0)}_k)_{\alpha(k)}\|_2\\
    &\overset{(c)}{\ge} \sqrt{2-2\left|\bip{(\D^{*}_k)_{\alpha(i)},(\D^{*}_k)_{\alpha(k)}}\right|}-2\max_{1\le i\le N} \epsilon_i\\
    &\overset{(d)}{\ge} \sqrt{2-2\frac{\mu}{\sqrt{d}}}-2\max_{1\le i\le N} \epsilon_i\\
    &\overset{(e)}{\ge} 2 \sum_{i=1}^{N} \epsilon^g_i
\end{aligned}
\end{equation}
Here $(a)$ and $(b)$ are due to Lemma~\ref{lemma: triangle for tilde d}, $(c)$ is due to assumed equality~\eqref{assumption:aligned}, $(d)$ is due to the $\mu$-incoherency of $\bD^*_k$, and finally $(e)$ is given by the assumption of Theorem~\ref{theorem: shortest path}.\par

{\bf Case 2:} Suppose $\alpha(i)> r^g$ for all $1\le i \le N$, which means that the path $\mathcal{P}$ only uses approximations to local atoms. Consider the ground truth of these approximations, $(\D^*_1)_{\alpha(1)},(\D^*_2)_{\alpha(2)},...,(\D^*_N)_{\alpha(N)}$. There must exist $1\le i,j\le N$ such that $d_{1,2}\bbracket{(\D^*_i)_{\alpha(i)},(\D^*_j)_{\alpha(j)}}\ge \beta$. Otherwise, it is easy to see that $\{\bD^{l*}_i\}_{i=1}^N$ would not be $\beta$-identifiable because any $(\D^*_i)_{\alpha(i)}$ will satisfy~\eqref{eq: condition no global candidate}. As a result, we have the following:
\begin{equation}
    \begin{aligned}
    \mathcal{L}(\mathcal{P})
    &\ge d_{1,2}\bbracket{(\D^{(0)}_i)_{\alpha(i)},(\D^{(0)}_j)_{\alpha(j)}}\\
    &\ge d_{1,2}\bbracket{(\D^*_i)_{\alpha(i)},(\D^*_j)_{\alpha(j)}} - \|(\D^*_i)_{\alpha(i)}-(\D^{(0)}_i)_{\alpha(i)}\|_2 - \|(\D^*_j)_{\alpha(j)} - (\D^{(0)}_j)_{\alpha(j)}\|_2\\
    &\ge \beta - 2\max_i \epsilon_i\\
    &\ge 2 \sum_{i=1}^{N} \epsilon_i
\end{aligned}
\end{equation}
So we have shown that $\{\mathcal{P}^*_j\}_{j=1}^{r^g}$ are the top-$r^g$ shortest paths from $s$ to $t$ in $\mathcal{G}$. Moreover, it is easy to show that $\mathrm{sign}\bbracket{\left\langle(\bD^{(0)}_1)_j,(\bD^{(0)}_i)_j\right\rangle} = 1$ for small enough $\{\epsilon_i\}_{i=1}^N$. Therefore, the proposed algorithm correctly recovers the global dictionaries (with the correct identity permutation). Finally, we have $\D^{g,(0)}=\frac{1}{N}\sum_{i=1}^N (\D^{(0)}_i)_{1:r^g}$, which leads to:
\begin{equation}
    \begin{aligned}
    d_{1,2}\left(\D^{g,(0)} , \rmD^{g*}\right)
    &\le \max_{1\le j \le r^g} \bnorm{\frac{1}{N}\sum_{i=1}^N (\D^{(0)}_i)_j-(\D^{g*})_j}_2\\ 
    &\le \max_{1\le j \le r^g} \frac{1}{N}\sum_{i=1}^N\bnorm{ (\D^{(0)}_i)_j-(\D^{g*})_j}_2\\
    &\le \max_{1\le j \le r^g} \frac{1}{N}\sum_{i=1}^N \epsilon_i\\
    &=\frac{1}{N}\sum_{i=1}^N \epsilon_i.
\end{aligned}
\end{equation}
This completes the proof of Theorem~\ref{theorem: shortest path}. $\hfill\square$
\subsection{Proof of Theorem~\ref{theorem: convergence PerMA}}
Throughout this section, we define:
\begin{equation}
    \Bar{\rho} := \frac{1}{N}\sum_{i=1}^N\rho_i,\quad\quad \Bar{\psi} := \frac{1}{N}\sum_{i=1}^N\psi_i.
\end{equation}
We will prove the convergence of the global dictionary in Theorem~\ref{theorem: convergence PerMA} by proving the following induction: at each $t\ge 1$, we have
\begin{equation}\label{eq: induction}
    d_{1,2}\bbracket{\bD^{g,(t+1)},\bD^{g*}}\le \Bar{\rho}d_{1,2}\bbracket{\bD^{g,(t)},\bD^{g*}}+ \bar{\psi}.
\end{equation}

At the beginning of communication round $t$, each client $i$ performs $\texttt{local\_update}$ to get $\bD^{(t+1)}_i$ given $\left[\bD^{g,(t)}\quad \bD^{l,(t)}_i\right]$. Without loss of generality, we assume
\begin{align}\label{assumption: aligned 2}
    \bI_{r_i\times r_i}&=\arg\min_{\bPi\in\mathcal{P}(r_i)}\bnorm{\bD^*_i\bPi-\left[\bD^{g,(t)}\quad \bD^{l,(t)}_i\right]}_{1,2},\\
    \bI_{r_i\times r_i}&=\arg\min_{\bPi\in\mathcal{P}(r_i)}\bnorm{\bD^*_i\bPi-\bD^{(t+1)}_i}_{1,2}.\label{assumption: aligned 3}
\end{align}
Assumed equalities \eqref{assumption: aligned 2} and \eqref{assumption: aligned 3} imply that the permutation matrix that aligns the input and the output of $\mathcal{A}_i$ is also $\bI_{r_i\times r_i}$. Specifically, the linear convergence property of $\mathcal{A}_i$ and Theorem~\ref{theorem: shortest path} thus suggest:
\begin{equation}
    \bnorm{\bbracket{\D^{(t+1)}_i}_j-\bbracket{\D^{*}_i}_j}_2\le \rho_i \bnorm{\bbracket{\D^{g,(t)}}_j-\bbracket{\D^{*}_i}_j}_2+\psi_i\quad\forall 1\le j\le r^g, 1\le i \le N.
\end{equation}
However, our algorithm is unaware of this trivial alignment. We will next show the remaining steps in $\texttt{local\_update}$ correctly recovers the identity permutation. The proof is very similar to the proof of Theorem~\ref{theorem: shortest path} since we are essentially running Algorithm~\ref{alg: shortest path} on a two-layer $\mathcal{G}$. For every $1\le i \le N$, $1\le j\le r^g$, we have 
\begin{equation}
    \begin{aligned}
    d_{1,2}\bbracket{\bbracket{\D^{(t+1)}_i}_j,\bbracket{\D^{g,(t)}}_j}&\le 
    d_{1,2}\bbracket{\bbracket{\D^{(t+1)}_i}_j,\bbracket{\D^{*}_i}_j} + d_{1,2}\bbracket{\bbracket{\D^{*}_i}_j,\bbracket{\D^{g,(t)}}_j}\\
    &\le 2\delta_i.
\end{aligned}
\end{equation}
Meanwhile for $k\ne j$,
\begin{equation}
    \begin{aligned}
    &d_{1,2}\bbracket{\bbracket{\D^{(t+1)}_i}_k,\bbracket{\D^{g,(t)}}_j}\\
    & \ge\ 
    d_{1,2}\bbracket{\bbracket{\D^{*}_i}_k,\bbracket{\D^{*}_i}_j}-d_{1,2}\bbracket{\bbracket{\D^{(t+1)}_i}_k,\bbracket{\D^{*}_i}_k} - d_{1,2}\bbracket{\bbracket{\D^{*}_i}_j,\bbracket{\D^{g,(t)}}_j}\\
    & \ge \sqrt{2-\frac{2\mu}{\sqrt{d}}}-2\delta_i.\\
    & \ge 2\delta_i.
\end{aligned}
\end{equation}
As a result, we successfully recover the identity permutation, which implies 
\begin{equation}
    \bnorm{\bbracket{\D^{g,(t+1)}_i}_j-\bbracket{\D^{g*}_i}_j}_2\le \rho_i \bnorm{\bbracket{\D^{g,(t)}}_j-\bbracket{\D^{g*}_i}_j}_2+\psi_i\quad\forall 1\le j\le r^g, 1\le i \le N.
\end{equation}
Finally, the aggregation step (Step~\ref{step: global aggregation} in Algorithm~\ref{alg: general alg}) gives:
\begin{equation}
    \begin{aligned}
    d_{1,2}\bbracket{\bD^{g,(t+1)},\bD^{g*}}&\le\bnorm{\frac{1}{N}\sum_{i=1}^N\D^{g,(t+1)}_i-\bD^{g*}}_{1,2} \\
    &=\max_{1\le j\le r^g} \bnorm{\bbracket{\frac{1}{N}\sum_{i=1}^N\D^{g,(t+1)}_i}_j-\bbracket{\bD^{g*}}_j}\\
    &\le \max_{1\le j\le r^g} \frac{1}{N}\sum_{i=1}^N\bnorm{\bbracket{\D^{g,(t+1)}_i}_j-\bbracket{\D^{g*}_i}_j}_2\\
    &\le \max_{1\le j\le r^g} \frac{1}{N}\sum_{i=1}^N\bbracket{\rho_i \bnorm{\bbracket{\D^{g,(t)}}_j-\bbracket{\D^{g*}_i}_j}_2+\psi_i}\\
    &\le  \frac{1}{N}\sum_{i=1}^N\bbracket{\rho_i d_{1,2}\bbracket{\bD^{g,(t)},\bD^{g*}}+\psi_i}\\
    &=\bar{\rho}d_{1,2}\bbracket{\bD^{g,(t)},\bD^{g*}}+\bar{\psi}.
\end{aligned}
\end{equation}
As a result, we prove the induction~\eqref{eq: induction} for all $0\leq t\leq T-1$. 
Inequality~\eqref{eq: local convergence} is a by-product of the accurate separation of global and local atoms and can be proved by similar arguments. The proof is hence complete. $\hfill\square$

\end{document}